\documentclass[letterpaper]{article}
\usepackage{proceed2e}
\usepackage[margin=1in]{geometry}

\usepackage{times}

\usepackage{microtype}
\usepackage{graphicx}
\usepackage{subfigure}
\usepackage{booktabs} % for professional tables

\usepackage{dsfont}

\usepackage{stackrel}
\usepackage{mathtools}
\usepackage{amsfonts}

\usepackage{tikz}
\usetikzlibrary{positioning}
\usetikzlibrary{decorations.text}
\usetikzlibrary{arrows}
\usetikzlibrary{shapes}

\usepackage{amssymb,amsmath,amsthm}

\usepackage{hyperref}

\usepackage{xcolor}

\newcommand{\cA}[0]{\mathcal{A}}
\newcommand{\cC}[0]{\mathcal{C}}
\newcommand{\cS}[0]{\mathcal{S}}
\newcommand{\expect}[1]{\mathbb{E}\left[ #1 \right]}

\newcommand{\alice}[0]{\textit{alice}}

\newtheorem{theorem}{Theorem}
\newtheorem{lemma}{Lemma}
\newtheorem{proposition}[theorem]{Proposition}
\newtheorem{definition}{Definition}
\newtheorem{observation}[theorem]{Remark}
\newtheorem{example}[theorem]{Example}

\newif\ifappendix
\appendixtrue

\title{PAC-Reasoning in Relational Domains}

\author{ {\bf Ond\v{r}ej Ku\v{z}elka} \\
Department of CS \\
KU Leuven\\
Leuven, Belgium \\
\And
{\bf Yuyi Wang}  \\
Disco Group\\
ETH Zurich          \\
Zurich, Switzerland \\
\And
{\bf Jesse Davis} \\
Department of CS \\
KU Leuven\\
Leuven, Belgium \\
\And
{\bf Steven Schockaert}   \\
School of CS \& Informatics \\
Cardiff University    \\
Cardiff, UK\\
}

\begin{document}

\maketitle

\begin{abstract}
We consider the problem of predicting plausible missing facts in relational data, given a set of imperfect logical rules. In particular, our aim is to provide bounds on the (expected) number of incorrect inferences that are made in this way. Since for classical inference it is in general impossible to bound this number in a non-trivial way, we consider two inference relations that weaken, but remain close in spirit to classical inference.
\end{abstract}

\section{INTRODUCTION}

In this paper we study several forms of logical inference for predicting plausible missing facts in relational data.
%in relational domains.
While a variety of approaches have already been studied for this task, ranging from (relational versions of) probabilistic graphical models \cite{Richardson2006,de2007problog} to neural-network architectures \cite{DBLP:conf/ilp/SourekMZSK16,DBLP:conf/nips/Rocktaschel017} and graph-based methods \cite{lao2011random,DBLP:conf/sigir/LiuJHLQ16}, logic-based inference has several advantages over these other forms of inference. For example, logic-based inference is explainable: there is a proof for any derived statement, which can, in principle, be shown to a human user. It is also more transparent than most other methods, in the sense that a knowledge base as a whole can be understood and modified by domain experts. On the other hand, classic logical inference can be very brittle when some of the rules which are used are imperfect, or some of the initial facts may be incorrect.
%only be used with data (evidence) that does not contradict it, i.e.\ the rules in the given theory have to be perfect.

Statistical relational learning approaches, such as Markov logic networks \cite{Richardson2006} or probabilistic logic programming \cite{de2007problog}, offer a solution to this latter problem, but they require
%A form of logic-based reasoning that can handle imperfect theories that may ``contradict'' given evidence is represented by statistical relational learning approaches such as Markov logic networks \cite{Richardson2006} or probabilistic logic programming approaches such as Problog \cite{de2007problog}. Both the statistical relational learning and probabilistic logic programming approaches require
learning a joint probability distribution over the set of possible worlds. This distribution is typically estimated from one or several large examples using maximum likelihood, which essentially corresponds to finding a maximum-entropy distribution given by a set of sufficient statistics. However, there are usually no guarantees on the learned distributions beyond guarantees for the sufficient statistics (see, e.g., \cite{kuzelka2018relational}), which means that we do not have much control over the quality of the predictions. Moreover, these models are not easy to modify, and are not always easy to explain
%to human users
because the way in which probabilities are computed can simply be too complex.

In this paper we focus on forms of inference that stay as close to classical logic as possible while not breaking completely when the given theory happens to be ``mildly'' inconsistent with the data. This problem of reasoning under inconsistency has a long tradition in the field of artificial intelligence, with common solutions including the use of paraconsistent logics \cite{da1974theory,priest1979logic}, belief revision \cite{Gardenfors} (and related inconsistency repair mechanisms \cite{konieczny2002merging}), and argumentation-based inference \cite{elvang1995argumentative,amgoud2014postulates}. In contrast to these approaches, however, our specific aim is to study forms of inference that can allow us to bound the (expected) number of mistakes that are made.
To this end, we introduce two inference relations called {\em $k$-entailment} and {\em voting entailment}, both of which are close to classical logic, and in particular do not require rules to be weighted. We define them such that errors produced by imperfect rules would not necessarily propagate too much in the given relational data.

As our main contribution, we are able to show that in a relational learning scenario from \cite{kuzelka2018relational}, in which a (large) training example and a test example are sampled from a hidden relational structure, there are non-trivial PAC-type bounds on the number of errors that a theory learned on the training example
%\nb{I guess this is misleading, as the bounds wouldn't be correct (I assume) if the formulas are learned from the training example?} -> No, the bounds have the size of the hypothesis space - i.e. the space from which we learn the theories - as a parameter. So it is actually correct.
produces on the test example. From this perspective, our work can also be seen as a relational-learning counterpart of PAC semantics~\cite{valiant_knowledge_infusion}.

\paragraph{Technical contributions.} The results presented in this paper rest mainly on the following two technical contributions:
%(i) introduction of the new inference relations and bounding of their worst case behavior,
(i) the introduction of bounds on the worst case behavior of the considered inference relations,
and (ii) new concentration inequalities for sampling from relational data without replacement that allow us to bound the (expected) test error as a function of the training error, in the spirit of classical PAC-learning results \cite{valiant1984theory}.

\section{PRELIMINARIES}

In this paper we consider a function-free first-order logic language $\mathcal{L}$, which is built from a set of constants $\textit{Const}$, variables $\textit{Var}$, and predicates $\textit{Rel} = \bigcup_i \textit{Rel}_i$, where $\textit{Rel}_i$ contains the predicates of arity $i$. We assume an untyped language. For $a_1,...,a_k \in \textit{Const}\cup \textit{Var}$ and $R \in \textit{Rel}_k$, we call $R(a_1,...,a_k)$ an atom.  If $a_1,..,a_k\in \textit{Const}$, this atom is called ground. A literal is an atom or its negation.
%A clause is a disjunction of literals.
%We use $vars(\alpha)$ to denote the variables that appear in a formula $\alpha$.
The formula $\alpha_0$ is called a grounding of $\alpha$ if $\alpha_0$ can be obtained by replacing each variable in $\alpha$ with a constant from $\textit{Const}$.
A formula is called closed if all variables are bound by a quantifier.
A possible world $\omega$ is defined as a set of ground atoms. The satisfaction relation $\models$ is defined in the usual way. A substitution is a mapping from variables to terms.

\section{PROBLEM SETTING}

First we describe the learning setting considered in this paper. It follows the setting from \cite{kuzelka2018relational},which was used to study the estimation of relational marginals.

An example is a pair $(\mathcal{A},\mathcal{C})$, with $\mathcal{C}$ a set of constants and $\mathcal{A}$ a set of ground atoms which only use constants from $\mathcal{C}$. An example is intended to provide a complete description of the world, hence any ground atom over $\mathcal{C}$ which is not contained in $\mathcal{A}$ is implicitly assumed to be false. Note that this is why we have to explicitly specify $\mathcal{C}$, as opposed to simply considering the set of constants appearing in $\mathcal{A}$.%\nb{S: I think we need a reference to the IJCAI paper here to acknowledge that we are following the basic setting from that paper.}.

In practice, we usually only have partial information about some example of interest.
The problems we consider in this paper relate to how we can then reason about the probability that a given ground atom is true (i.e.\ belongs to the example). To estimate such probabilities, we assume that we are given a fragment of the example, which we can use as training data. Specifically,
let $\Upsilon = (\mathcal{A},\mathcal{C})$ be an example and $\mathcal{S}\subseteq \mathcal{C}$. The fragment $\Upsilon\langle S \rangle = (\mathcal{B},\mathcal{S})$ is defined as the restriction of $\Upsilon$ to the constants in $\mathcal{S}$, i.e.\ $\mathcal{B}$ is the set of all atoms from $\mathcal{A}$ which only contain constants from $\mathcal{S}$.
In a given example, any closed formula $\alpha$ is either true or false. To assign probabilities to formulas in a meaningful way, we consider how often the formula is satisfied in small fragments of the given example.%\nb{Again, we probably need to cite the IJCAI paper here, as now it seems as if this is something new that we propose in this paper.}.

\begin{definition}[Probability of a formula \cite{kuzelka2018relational}]\label{def:probability_of_formula}
Let $\Upsilon = (\mathcal{A},\mathcal{C})$ be an example and $k\in \mathbb{N}$.
For a closed formula $\alpha$ without constants, we define its probability as follows\footnote{We will use $Q$ for probabilities of formulas as defined in this section, to avoid confusion with other ``probabilities'' we deal with in the text.}:
$$Q_{\Upsilon,k}(\alpha) = P_{\mathcal{S} \sim \textit{Unif}(\cC,k)}\left[ \Upsilon\langle \mathcal{S} \rangle \models \alpha \right]$$
where $\textit{Unif}(\cC,k)$ denotes uniform distribution on size-$k$ subsets of $\cC$.
\end{definition}

\noindent Clearly $Q_{\Upsilon,k}(\alpha) = \frac{1}{|\mathcal{C}_k|} \cdot \sum_{\mathcal{S} \in \cC_k} \mathds{1}(\Upsilon\langle \mathcal{S} \rangle \models \alpha)$ where $\mathcal{C}_k$ is the set of all size-$k$ subsets of $\cC$.

The above definition is also extended straightforwardly to probabilities of sets of formulas (which we will also call {\em theories} interchangeably). If $\Phi$ is a set of formulas, we set
%\nb{Why use $\myeq$ here (but not elsewhere in the paper)?}
$Q_{\Upsilon,k}(\Phi) = Q_{\Upsilon,k}(\bigwedge \Phi)$ where $\bigwedge \Phi$ denotes the conjunction of all formulas in $\Phi$.

\begin{example}
Let $\textit{sm}/1$ be a unary predicate denoting that someone is a smoker, e.g. $\textit{sm}(\textit{alice})$ means that $\alice$ is a smoker. Let us have an example
$\Upsilon = (\{ \textit{fr}(\textit{alice}, \textit{bob}), \textit{sm}(\textit{alice}), \textit{sm}(\textit{eve}) \},
    \{\textit{alice}, \textit{bob}, \textit{eve} \}),$
and formulas $\alpha = \forall X : \textit{sm}(X)$ and $\beta = \exists X,Y : \textit{fr}(X,Y)$. Then, for instance, $Q_{\Upsilon,1}(\alpha) = 2/3$, $Q_{\Upsilon,2}(\alpha) = 1/3$ and $Q_{\Upsilon,2}(\beta) = 1/3$.
\end{example}

\begin{definition}[Masking]
A masking process is a function $\kappa$ from examples to ground conjunctions that assigns to any $\Upsilon = (\cA,\cC)$ a conjunction of ground literals $\beta$ such that $\Upsilon \models \beta$. We also define $\kappa(\Upsilon)\langle \cS \rangle$ to be the conjunction consisting of all literals from $\kappa(\Upsilon)$ that contain only constants from $\cS$.
\end{definition}

Unlike examples, masked examples only encode partial information about the world. This is why they are encoded using conjunctions of literals, so we can explicitly encode which atoms we know to be false.
%Since, unlike an example, a conjunction is not necessarily a complete description, masked examples may represent partial knowledge.

\begin{example}
Let $\Upsilon = \{ \textit{sm}(\textit{alice}),$ $\textit{fr}(\textit{alice},\textit{bob}),$ $\{\textit{alice},\textit{bob}\} \}.$ Then a masking process $\kappa$ may, for instance, yield $\kappa(\Upsilon) = \neg \textit{sm}(\textit{bob}) \wedge \textit{sm}(alice)$. In this case $\kappa(\Upsilon)$ retains the information that $\textit{alice}$ is a smoker and $\textit{bob}$ is not, but it no longer contains any information about their friendship relation.
\end{example}

Next we introduce the statistical setting considered in this paper.

\begin{definition}[Learning setting]\label{def:inductive_setting}
Let $\aleph = (\cA_\aleph, \cC_\aleph)$ be an example and $\kappa$ be a masking function. Let $\cC_\Upsilon \subseteq \cC_\aleph$ and $\cC_\Gamma \subseteq \cC_\aleph$ be uniformly sampled subsets of size $n$ and $u$, respectively. We call $\Upsilon = \aleph \langle \cC_\Upsilon \rangle$ the {\em training example} and $\Gamma = \aleph \langle \cC_\Gamma \rangle$ the {\em test example}. %\nb{Removed ``In this setting''}
We assume that the learner receives $\Upsilon$ in the training phase and $\kappa(\Gamma)$ in the test phase.
\end{definition}

With slight abuse of terminology, we will sometimes say that $\Upsilon$ and $\Gamma$ are sampled from $\aleph$.

In addition to the training example $\Upsilon$ and masked test example $\kappa(\Gamma)$, we will assume that we are given a set of formulas $\Phi$ (which we will also refer to as rules). Our main focus will be on how these formulas can be used to recover as much of $\Gamma$ as possible. Rather than specifying a loss function that should be minimized, we want to find a form of inference which allows us to provide bounds on the (expected) number of incorrect literals that can be inferred from $\{\kappa(\Gamma)\} \cup \Phi$. Note that in this case, the training example $\Upsilon$ is used to estimate the accuracy of the set of formulas.
%Here we are given a set of logic formulas (``rules'') and the goal is to recover as much of $\Gamma$ as possible using these formulas without making too many wrong inferences.
%At this point, one could wonder why we did not define a loss function that should be minimized. The main reason is that in this paper, we do not directly try to minimize one. %Instead we want to find a way to use the given formulas while being able to bound the number of errors that we will make.
%To be able to bound this number, we first estimate the accuracies of the formulas on the training data $\Upsilon$ and then perform the inference on $\kappa(\Gamma)$.
We also analyze the case where the rules are learned from the training example $\Upsilon$ (in the spirit of classical PAC-learning results).
%\nb{Removed suggestion that we will also consider the case where rules are learned from the training example.}

Among others, the setting from Definition \ref{def:inductive_setting} is close to how Markov logic networks are typically used. For instance, when training Markov logic networks, one typically starts with a training example that contains all facts (i.e.\ nothing is unknown about the training set), on which a model is trained. This model is then used to predict unknown facts about a test example. However, unlike for Markov logic networks, we do not attempt to learn a probability distribution. It was shown in \cite{kuvzelka2017induction} that models based on classic logical inference, like those considered in this paper, work well in practice for relational inference from evidence sets containing a {\em small} number of constants (domain elements). Thus, such models are also of considerable practical interest.

\section{REASONING WITH INACCURATE RULES}

When reasoning with imperfect rules, using classical inference can have drastic consequences, as we will illustrate in Section \ref{secClassicalErrors}. Even a single mistake can lead to many errors, since an incorrectly derived literal can be used as the basis for further inferences. This means that classical inference is not suitable for the considered setting, even in cases where the given rules have perfect accuracy on the training example. Intuitively, to allow for any meaningful bounds to be derived, we need to prevent arbitrarily long chains of inference. To this end, we propose and motivate the use of a restricted form of inference, called $k$-entailment, in Section \ref{sec:k-entailment}. A further restriction on inferences, based on a form of voting, is subsequently discussed in Section \ref{sec:voting-entailment}. In Section \ref{sec:pac} we will then show which bounds can be derived for these two restricted forms of inference.

\subsection{WHEN CLASSICAL REASONING LEADS TO ERRORS}\label{secClassicalErrors}

The next example, which is related to label propagation as studied e.g.\ in \cite{xiang2011understanding}, shows that classic logical reasoning on the obtained relational sample may produce
%{\em big}
{\em many}
mistakes even when all the available rules are very {\em accurate}.
%\nb{Should we follow Jesse's suggestion here and acknowledge that this is well-known (and e.g.\ cite that early SRL paper he mentioned)?}

\begin{example}\label{ex:mistakes1}
Let $k = 2$, $\Gamma = \{ \{ \textit{rare}(c_1) \},$ $\{c_1,$ $c_2,$ $\dots,$ $c_{1000000} \}$, and $\alpha = \forall X,Y : \textit{rare}(X) \Rightarrow \textit{rare}(Y)$. While the rule does not intuitively make sense,
%\nb{Just as an aside, we should really do something sometime that makes explicit a preference for positive literals, as it's this preference that makes the rule here unintuitive (i.e.\ it's accurate but it mostly predicts positive facts). Maybe it's as simple as replacing accuracy by F1 score?}
its accuracy is actually very high $Q_{\Gamma,k}(\alpha) = 1- 999999/(0.5 \cdot 1000000 \cdot 999999) = 0.999998$.
When we apply this rule with the evidence $\textit{rare}(c_1)$,
%which we could obtain e.g. from an adversarial masking process which selected only this positive evidence and no negative evidence,
we derive $\textit{rare}(c_2)$, $\dots$, $\textit{rare}(c_{1000000})$, all of which are incorrect (i.e.\ not included in $\Gamma$).
\end{example}
Note that in this paper, we are interested in worst-case behavior, in the sense that the masking process which is used may be seen as adversarial. The next example further illustrates how adversarial masking processes can lead to problems, even for rules with near-perfect accuracy.
%In fact, the above example is not even the worst case as we illustrate next.

\begin{example}\label{example:mistakes2}
Let $k = 2$, $\Gamma = \{ \{ \textit{rare}(c_1)$, $e(c_1,c_2)$, $e(c_2,c_3)$, $\dots$, $e(c_{999999},c_{1000000})$ $ \},$ $\{c_1,$ $c_2,$ $\dots,$ $c_{1000000} \}$, and $\alpha = \forall{X,Y} : \textit{rare}(X) \wedge e(X,Y) \Rightarrow \textit{rare}(Y)$. In this case, there is only one size-$k$ subset of $C_\Gamma$ where the formula $\alpha$ does not hold, so the accuracy is even higher than in the previous example. Yet the adversarial masking process can select evidence consisting of all true positive literals from $\Gamma$, i.e. the evidence will consist of the $\textit{rare}(c_1)$ literal and all the $e/2$ literals from $\Gamma$. Then the set of errors that are made when using the formula $\alpha$ will be the same as in Example \ref{ex:mistakes1}, despite the fact that the rule is almost perfect on $\Gamma$.
\end{example}

Note that in the examples above, we had perfect knowledge of the accuracy of the rule $\alpha$ on the test example (i.e.\ we knew the value of $Q_{\Gamma,k}(\alpha)$). In practice, this accuracy needs to be estimated from the training example. In such cases, it can thus happen that a rule $\alpha$ has accuracy $1$ on the training example $\Upsilon$, but still produces many errors on $\kappa(\Gamma)$. We will provide PAC-type bounds for this setting with estimated accuracies in Sections \ref{sec:pac}. First, however, in Section \ref{sec:k-entailment} and \ref{sec:voting-entailment} we will look at how bounds can be provided on the number of incorrectly derived literals in the case where $Q_{\Gamma,k}(\alpha)$ is known. As the above examples illustrate, to obtain reasonable bounds, we will need to consider forms of inference which are weaker than classical entailment.

\subsection{BOUNDED REASONING USING $k$-ENTAILMENT}\label{sec:k-entailment}

We saw that even for formulas which hold for almost all subsets of $\Gamma$, the result of using them for inference can be quite disastrous. This was to a large extent due to the fact that we had
%the fact that, by using classical logic inference, we allowed
inference chains involving a large number of domain elements (constants). This observation suggests a natural way to restrict the kinds of inferences that can be made when imperfect rules are involved.

%we are not certain that the given rules always hold.

\begin{definition}[$k$-entailment]
Let $k$ be a non-negative integer, $\Upsilon = (\cA, \cC)$ be an example, $\kappa$ be a masking process, and $\Phi$ be a set of closed formulas. We say that a ground formula $\varphi$ is $k$-entailed by $\Phi$ and $\kappa(\Upsilon)$, denoted $\{\kappa(\Upsilon)\} \cup \Phi \models_k \varphi$, if there is a $\mathcal{C}' \subseteq \cC$ such that $|\mathcal{C}'| \leq k$, $\textit{const}(\varphi) \subseteq \cC'$, $\{\kappa(\Upsilon)\langle \mathcal{C}' \rangle \} \cup \Phi$ is consistent and $\{ \kappa(\Upsilon)\langle \mathcal{C}' \rangle \} \cup \Phi \models \varphi$.
\end{definition}

In other words, a formula $\phi$ is $k$-entailed by $\Upsilon$ and $\Phi$ if it can be proved using $\Phi$ together with a fragment of $\kappa(\Upsilon)$ induced by no more than $k$ constants, with the additional condition that $\Phi$ and this fragment are not contradictory.

\begin{example}
Let
\begin{align*}
    \Upsilon &= (\{ \textit{fr}(\textit{alice}, \textit{bob}), \textit{sm}(\textit{alice}) \}, \{\textit{alice}, \textit{bob}, \textit{eve} \})\\
    \kappa(\Upsilon) &= \textit{fr}(\textit{alice} \wedge \textit{bob}) \wedge \textit{sm}(\textit{alice}) \\
    \Phi &= \{ \forall X,Y \colon \textit{fr}(X,Y) \wedge \textit{sm}(X) \Rightarrow \textit{sm}(Y) \}.
\end{align*}
Then $\varphi = \textit{sm}(bob)$ is $2$-entailed from $\kappa(\Upsilon)$ and $\Phi$ but not $1$-entailed.
\end{example}

Note that, in the setting of Example \ref{example:mistakes2}, $k$-entailment would make at most $k-1$ mistakes. However, $2$-entailment would already produce many mistakes in the case of Example \ref{ex:mistakes1}. So there are cases where $k$-entailment produces fewer errors than classical logic entailment but, quite naturally, also cases where both produce the same number of errors. Importantly, however, for $k$-entailment, we can obtain non-trivial bounds on the number of errors.

Next we state two lemmas that follow immediatelly from the respective definitions. %The next lemma formalizes the intuition that $k$-entailment does not allow inference chains over many domain elements.

\begin{lemma}\label{lemma:union}
Let $\Upsilon = (\cA,\cC)$ be an example, $\Phi$ be a set of constant-free formulas and $\kappa$ be a masking function. Let $\mathcal{C}_k$ be the set of all size-$k$ subsets of $\mathcal{C}$. Let $\mathcal{H}_{\mathcal{X}}$ denote the set of all ground literals which can be derived using $k$-entailment from $\{ \kappa(\Upsilon)\} \cup \Phi$ and only contain constants from $\mathcal{X}$.
%all ground literals containing only constants from the set $\mathcal{X}$ that can be derived using $k$-entailment from $\{ \kappa(\Upsilon)\} \cup \Phi$.
Then
$\mathcal{H}_\mathcal{C} = \bigcup_{\mathcal{S} \in \cC_l} \mathcal{H}_\mathcal{S}. $
\end{lemma}
% \begin{proof}
% Follows from the definitions.
% \end{proof}

\begin{lemma}\label{remark:monotonicity}
%\nb{S: I know this property is used further on in the proof, but here it feels out of place. While it's not hard to see that this property holds, it's unclear why it's important enough to make a comment about. I would make this observation as part of the relevant proof.}
When $\Gamma\langle \mathcal{S} \rangle \models \Phi$ then all ground literals that only contain constants from $\mathcal{S}$ and that are entailed by $\{\kappa(\Gamma\langle \mathcal{S} \rangle)\} \cup \Phi$ must be true in $\Gamma \langle \mathcal{S} \rangle$.
%\nb{S: deleted ``which follows from monotonicity of the entailment relation $\models$'', as I think this statement might just cause confusion.}.
\end{lemma}
% \begin{proof}
% Follows straightforwardly.
% \end{proof}

\noindent We now provide a bound on the number of ground literals wrongly $k$-entailed by a given $\Phi$, assuming that we know its accuracy $Q_{\Gamma,k}(\Phi)$ on the example $\Gamma$.%\nb{Instead of writing accuracy between quotes, could we write ``its accuracy $Q_{\Gamma,k}(\Phi)$ on the example''?}.

\begin{proposition}\label{prop:simplebound}
Let $\Gamma = (\cA,\cC)$ be an example, $\Phi$ be a set of constant-free formulas and $\kappa$ be a masking process. Next let $\mathcal{F}(\Gamma)$ be the set of all ground literals of a predicate $p/a$, $a\leq k$, which are $k$-entailed by $\{\kappa(\Gamma)\} \cup \Phi$ but are false in $\Gamma$. Then
$$|\mathcal{F}(\Gamma)| \leq (1-Q_{\Gamma,k}(\Phi)) |\cC|^k k^a.$$
\end{proposition}
\begin{proof}
First, we note that the number of size-$k$ subsets is bounded by $|\cC|^k$ and the number of different ground $p/a$ atoms in each of these subsets is $k^a$.
It follows from Lemma \ref{remark:monotonicity} and Lemma \ref{lemma:union} that for any literal $\delta \in \mathcal{F}$ there must be a size-$k$ set $\mathcal{S} \subseteq \cC$ such that $\Gamma\langle \cS \rangle \not\models \Phi$. The number of all such $\cS$'s that satisfy $\Gamma\langle \cS \rangle \not\models \Phi$ is bounded by $(1-Q_{\Gamma,k}(\Phi)) |\cC|^k$.
Hence, we have $|\mathcal{F}(\Gamma)| \leq (1-Q_{\Gamma,k}(\Phi)) |\cC|^k k^a.$
\end{proof}

We can notice that when we increase the domain size $|\cC|$, keeping $Q_{\Gamma,k}(\Phi)$ fixed and non-zero, the bound eventually becomes vacuous for predicates whose arity $a$ is strictly smaller than $k$. This is because the number of all ground literals grows only as $|\cC|^a$ whereas the bound grows as $|\cC|^k$. However, %this does not hold for literals with arity equal to $k$; for literals of arity equal to
if $a=k$, the bound stays fixed when we increase the domain size. We will come back to consequences of this fact in Section \ref{sec:results}.

\subsection{BOUNDED REASONING USING VOTING}\label{sec:voting-entailment}

To further restrict the set of entailed ground literals, we next introduce {\em voting entailment}.

\begin{definition}[Voting Entailment]
Let $k$ be an integer and $\gamma \in [0;1]$. Let $\Upsilon = (\cA,\cC)$ be an example, $\Phi$ be a set of constant-free formulas, and $\kappa$ be a masking process. A ground literal $l$ of arity $a$, $a\leq k$, is said to be entailed from $\Phi$ and $\kappa(\Upsilon)$ by voting with parameters $k$ and $\gamma$ if there are at least $\max\{1,\gamma \cdot |\cC|^{k-a}\}$ size-$k$ sets $\cS \subseteq \cC$ such that $l$ is $k$-entailed by $\kappa(\Upsilon)\langle \cS \rangle$.
\end{definition}
%\nb{Add intuition for why the number of sets needs to be $\max\{1,\gamma \cdot |\cC|^{k-a}\}$. Now it looks like the definition was reverse-engineered from the bounds that are shown further on.} - But that is actually correct - we do not want something that comes from some postulates but something that works..
The next example illustrates the use of voting entailment.

\begin{example}
Let $\Upsilon = (\cA,\cC)$, where $\cC = \{\textit{alice},\textit{bob},\textit{eve}\}$, and let $\kappa(\Upsilon) = \textit{fr}(\textit{alice},\textit{bob})$ $\wedge$ $\textit{fr}(\textit{eve},\textit{bob})$ $\wedge$ $\textit{sm}(\textit{eve})$. Next, let $\Phi = \{ \forall X,Y : \textit{fr}(X,Y) \wedge \textit{sm}(X) \Rightarrow \textit{sm}(Y) \}$. Then $\textit{sm}(\textit{bob})$ is entailed from $\Phi$ and $\kappa(\Upsilon)$ by voting with the parameters $k = 2$ and $\gamma = 2/3$, as $\gamma \cdot |\cC|^{k-a} = 2/3 \cdot 3^{2-1} = 2$ and there are two size-$2$ subsets of $\cC$ that $2$-entail $\textit{sm}(\textit{bob})$.
\end{example}

% In the previous section, we saw that for predicates of arity $k$, we have non-vacuous bounds on the number of errors produced by $k$-entailment. The next observation is important in light of this as it shows that predictions of voting entailment agree with predictions of $k$-entailment for literals of arity $k$.

% \begin{observation}
% Any literal that has arity $k$ and is $k$-entailed by given $\{\kappa(\Gamma)\} \cup\Phi$ is also entailed from $\{ \kappa(\Gamma)\} \cup\Phi$ by voting (with parameters $k' \geq k$ and $0 \leq \gamma \leq 1/{k'}^{k'}$).
% \end{observation}

We now show how the bound from Proposition \ref{prop:simplebound} can be strengthened in the case of voting entailment.
\begin{proposition}\label{prop:voting}
Let $k$ be an integer and $\gamma \in [0;1]$. Let $\Gamma = (\cA,\cC)$ be an example, $\Phi$ be a set of constant-free formulas, and $\kappa$ be a masking process. Let $\mathcal{F}(\Gamma)$ be the set of all ground literals of a predicate $p/a$, $a\leq k$, that are entailed by voting from $\{\kappa(\Gamma)\} \cup \Phi$ with parameters $k$ and $\gamma$ but are false in $\Gamma$. If $\gamma \cdot |\cC|^{k-a} \geq 1$ then
\begin{align*}
    |\mathcal{F}(\Gamma)| \leq \left( 1-Q_{\Gamma,k}(\Phi) \right) \frac{|\cC|^{a} k^a}{\gamma}
\end{align*}
and otherwise
\begin{align*}
    |\mathcal{F}(\Gamma)| \leq \left( 1-Q_{\Gamma,k}(\Phi) \right) |\cC|^{k} k^a.
\end{align*}
\end{proposition}
\begin{proof}
First we define the number of ``votes'' for a ground literal $l$ as
$$
    \#_{\kappa(\Gamma),\Phi}(l) = |\left\{\cS \subseteq \cC \left| |\cS| {=} k, \{\kappa(\Gamma)\langle \cS \rangle\} \cup \Phi \models_k l \right. \right\}|.
$$
Let $L$ be the set of all ground $p/a$ literals $l$ such that $\Gamma \models \neg l$. Then, since any size-$k$ subset of $\cC$ can only contribute $k^a$ votes to literals based on the predicate $p/a$,  we have
$$\sum_{l \in L} \#_{\kappa(\Gamma),\Phi}(l) \leq \left(1-Q_{\Gamma,k}(\Phi) \right) |\cC|^k k^a.$$
Hence
$|\mathcal{F}(\Gamma)| \leq \frac{\left( 1-Q_{\Gamma,k}(\Phi) \right) |\cC|^{k} k^a}{\max\{1, \gamma\cdot |\cC|^{k-a} \}}. $
If $\gamma \cdot |\cC|^{k-a} \geq 1$ then $|\mathcal{F}(\Gamma)| \leq \left( 1-Q_{\Gamma,k}(\Phi) \right) \frac{|\cC|^{a} k^a}{\gamma}$.
The case when $\gamma \cdot |\cC|^{k-a} < 1$ follows from Theorem~\ref{prop:simplebound}.
\end{proof}

%Note that,
Unlike for $k$-entailment, the fraction of ``wrong'' ground $p/a$ literals entailed by voting entailment does not grow with an increasing domain size as long as $\gamma \cdot |\cC|^{k-a} \geq 1$.

% \subsection{GETTING CONSISTENT PREDICTIONS}

% %We note here that there is one counter-intuitive aspect of $k$-entailment and voting entailment which is that the set of derived literals may be inconsistent, either because it contains a literal and its negation. This can be solved simply by finding the smallest subset of derived literals that need to be removed to get rid of all inconsistencies. In this way no new incorrectly derived literals are derived and the bounds from the preceding propositions still hold.

% \todo{}

\section{PROBABILISTIC BOUNDS}\label{sec:pac}

%We now turn to the setting where the accuracy of the given formulas and possibly the formulas themselves need to be estimated from a training example. Specifically we prove probabilistic bounds for the following learning problem. We are given a hypothesis set $\mathcal{H}$ of constant-free theories, and we want to simultaneously bound the number of literals incorrectly predicted by the individual theories from $\mathcal{H}$ as a function of the following parameters: errors of the theories on the training example, the size of $\mathcal{H}$ and the size of the domain of the training example and the size of the test example.
We now turn to the setting where the accuracy of the formulas
%, and possibly the formulas themselves,
needs to be estimated from a training example $\Upsilon$.
More generally, we also cover the case where the formulas themselves are learned from the training example. In such cases, to account for over-fitting, we need to consider the (size of the) hypothesis class that was used for learning these formulas.
Specifically, we prove probabilistic bounds for variants of the following learning problem.
We are given a hypothesis set $\mathcal{H}$ of constant-free theories, and
%, for all of the individual theories $\Phi \in \mathcal{H}$,
we want to compute bounds on the number of incorrectly predicted literals which simultaneously hold for all $\Phi \in \mathcal{H}$ (as a function of $Q_{\Upsilon,k}(\Phi)$) with probability at least $1-\delta$, where $\delta$ is a confidence parameter. Note that the case where the theory $\Phi$ is given, rather than learned, corresponds to $\mathcal{H}=\{\Phi\}$.

%the number of literals incorrectly predicted by each of the individual theories $\Phi \in \mathcal{H}$ as a function of the following parameters: error of $\Phi$ on the training example, the size of $\mathcal{H}$ and the size of the domain of the training example and the size of the domain of the masked test example.

We start by proving general concentration inequalities in Section \ref{sec:concentration} which we then use to prove bounds for $k$-entailment. These bounds are studied for the realizable case in Section \ref{sec:realizable} and for the general case in Section \ref{sec:non-zero}. Bounds for voting entailment are studied in Section \ref{sec:voting-bounds}

%\todo{glue - the bounds we prove here are for all theories from a given hypothesis class - if it contains just one hypothesis then this is important special case too, otherwise the bounds can be thought of as bounds for learning from a finite hypothesis class... first we prove general concentration inequalities and then use them to bound the number of errors produced by k-entailment and voting entailment}

\subsection{CONCENTRATION INEQUALITIES}\label{sec:concentration}

We will need to bound the difference between the ``accuracy''\ of given sets of logic formulas $\Phi$ on the training sample $\Upsilon$ and their accuracy on a test sample $\Gamma$ (i.e.\ the difference between $Q_{\Upsilon,k}(\Phi)$ and $Q_{\Gamma,k}(\Phi)$).
To prove the concentration inequalities in this section, we will utilize the following lemma.

\begin{lemma}[Ku\v{z}elka et al. \cite{kuzelka2018relational}]\label{lemma:crazy1}
Let $\aleph = (\cA_\aleph, \cC_\aleph)$ be an example. Let $0 \leq n \leq |\cC_\aleph|$ and $0 \leq k \leq n$ be integers. Let $\mathbf{X} = (\cS_1,\cS_2,\dots,\cS_{\lfloor \frac{n}{k} \rfloor})$ be a vector of subsets of $\cC_\aleph$, each sampled uniformly and independently of the others from all size-$k$ subsets of $\cC_\aleph$. Next let $\mathcal{C}_\Upsilon$ be sampled uniformly from all size-$n$ subsets of $\mathcal{C}_\aleph$. Finally, let $\mathcal{I}' = \{ 1,2,\dots,|\cC_\aleph| \}$ and let $\mathbf{Y} = (\cS_1',\cS_2',\dots,\cS_{\lfloor \frac{n}{k} \rfloor}')$ be a vector sampled by the following process:
\begin{enumerate}
    %\item Sample $\mathcal{C}_\Upsilon$ uniformly from all size-$n$ subsets of $\mathcal{C}_\aleph$.
    \item Sample subsets $\mathcal{I}_1',\dots,\mathcal{I}_{\lfloor \frac{n}{k} \rfloor}'$ of size $k$ from $\mathcal{I}'$.
    \item Sample an injective function $g : \bigcup_{i=1}^{\lfloor n/k \rfloor} \mathcal{I}_i' \rightarrow \cC_\Upsilon$ uniformly from all such functions.
    \item Define $\cS_i' = g(\mathcal{I}_i')$ for all $0 \leq i \leq \lfloor \frac{n}{k} \rfloor$.
\end{enumerate}
Then $\mathbf{X}$ and $\mathbf{Y}$ have the same distribution.
\end{lemma}

%The lemma formalizes the intuition that we get more information if we sample a fragment on $n$ constants than if we sample $\lfloor n/k \rfloor$ fragments on $k$ constants ($k \leq n$).
The next example illustrates the intuition behind the proof of this lemma, which can be found in \cite{kuzelka2018relational}.

\begin{example}
Let $\cC_\aleph = \{1,2,\dots, 10^6\}$. Let us sample $\lfloor m/k \rfloor$ size-$k$ subsets of $\cC_\aleph$ uniformly. If this was the process that generates the data from which we estimate parameters, we could readily apply Hoeffding's inequality to get the confidence bounds. However, in typical SRL settings (e.g.\ with MLNs), we are given a complete example on some set of constants (objects), rather than a set of small sampled fragments. So we instead need to assume that the whole training example is sampled at once, uniformly from all size-$m$ subsets of $\cC_\aleph$. However, when we then estimate the probabilities of formulas from this example, we cannot use Hoeffding's bound or any other bound expecting independent samples. What we can do\footnote{Note that we do not need to do this in practice which will follow from Theorem \ref{prop:u-concentration}; we only need this mimicking process to prove that theorem.} is to mimic sampling from $\cC_\aleph$ by sampling from an auxiliary set of constants of the same size as $\cC_\aleph$ and then specialising these constants to constants from a sampled size-$m$ subset. Hence the first $\lfloor m/k \rfloor$ sampled sets will be distributed exactly as the first $\lfloor m/k \rfloor$ subsets sampled i.i.d. directly from $\cC_\aleph$.
\end{example}

% The next lemma is a direct application of Chernoff-Hoeffding's theorem \cite{hoeffding} and Lemma \ref{lemma:crazy1}.

% \begin{lemma}[Ku\v{z}elka et al. \citeyear{kuzelka.aaai.2018}]\label{lemma:hoeffding1}
% Let $\aleph = (\cA_\aleph, \cC_\aleph)$ be an example and let $0 \leq n \leq |\cC_\aleph|$ and $0 \leq k \leq n$ be integers. Let $\mathcal{C}_\Upsilon$ be sampled uniformly from all size-$n$ subsets of $\mathcal{C}_\aleph$ and let $\Upsilon = \aleph\langle \cC_\Upsilon \rangle$. Let $\mathbf{Y}$ be sampled as in Lemma \ref{lemma:crazy1} (i.e.\ $\mathbf{Y}$ is sampled only using $\Upsilon$ and not directly $\aleph$). Let $\alpha$ be a closed and constant-free formula. Let $\widetilde{A}_\Upsilon = Q_{\Upsilon,k}(\alpha)$
% % $$\widetilde{A}_\Upsilon = \frac{1}{\lfloor \frac{n}{k} \rfloor} \sum_{\mathcal{S}_i' \in \mathbf{Y}} \mathds{1}(\Upsilon\langle \mathcal{S}_i' \rangle \models \alpha)$$
% and let $A_{\aleph} = Q_{\aleph,k}(\alpha)$. Then we have
% $$P\left[ \left| \widetilde{A}_\Upsilon - A_\aleph \right| \geq \varepsilon \right] \leq 2 \exp\left(-2 \left\lfloor \frac{n}{k} \right\rfloor \varepsilon^2 \right).$$
% \end{lemma}

Lemma \ref{lemma:crazy1} was used in \cite{kuzelka2018relational} to prove a bound on expected error. Here we extend that result and use Lemma \ref{lemma:crazy1} to prove the concentration inequalities stated in the next two theorems.

\begin{theorem}\label{prop:u-concentration}
Let $\aleph = (\cA_\aleph, \cC_\aleph)$ be an example and let $0 \leq n \leq |\cC_\aleph|$ and $0 \leq k \leq n$ be integers. Let $\mathcal{C}_\Upsilon$ be sampled uniformly from all size-$n$ subsets of $\mathcal{C}_\aleph$ and let $\Upsilon = \aleph\langle \cC_\Upsilon \rangle$. Let $\alpha$ be a closed and constant-free formula and let $\cC_k$ denote all size-$k$ subsets of $\cC_\Upsilon$. Let
%$$\widehat{A}_\Upsilon = \left( \begin{array}{c} n \\ k \end{array} \right)^{-1} \sum_{\cS \in \cC_k} \mathds{1}(\Upsilon\langle \cS \rangle \models \alpha)$$
$\widehat{A}_\Upsilon = Q_{\Upsilon,k}(\alpha)$
and let $A_{\aleph} = Q_{\aleph,k}(\alpha)$. Then we have
$P[ \widehat{A}_\Upsilon - A_\aleph \geq \varepsilon ] \leq \exp\left(-2 \left\lfloor \frac{n}{k} \right\rfloor \varepsilon^2 \right)$, $P[ A_\aleph-\widehat{A}_\Upsilon  \geq \varepsilon ] \leq \exp\left(-2 \left\lfloor \frac{n}{k} \right\rfloor \varepsilon^2 \right)$,
and
$P\left[ \left| \widehat{A}_\Upsilon - A_\aleph \right| \geq \varepsilon \right] \leq 2 \exp\left(-2 \left\lfloor \frac{n}{k} \right\rfloor \varepsilon^2 \right).$
\end{theorem}
\begin{proof}
First we define an auxiliary estimator $\widetilde{A}^{(q)}_{\Upsilon}$. Let $\mathbf{Y}^{(q)}$ be a vector of $\lfloor n/k \rfloor \cdot q$ size-$k$ subsets of $\cC_\Upsilon$ where the subsets of $\cC_{\Upsilon}$ in each of the $q$ non-overlapping size-$\lfloor n/k \rfloor$ segments $\mathbf{Y}_1^{(q)}, \mathbf{Y}_2^{(q)},\dots,\mathbf{Y}_q^{(q)}$ of $\mathbf{Y}^{(q)}$ are sampled in the same way as the elements of the vector $\mathbf{Y}$ in Lemma \ref{lemma:crazy1}, all with the same $\cC_\Upsilon$ (i.e. $\mathbf{Y}^{(q)}$ is the concatenation of the vectors $\mathbf{Y}^{(q)}_1,\mathbf{Y}^{(q)}_2,\dots,\mathbf{Y}^{(q)}_q$). Let us define
$\widetilde{A}^{(q)}_{\Upsilon} = \frac{1}{q \cdot \lfloor n/k \rfloor} \sum_{\cS \in \mathbf{Y}^{(q)}} \mathds{1}(\Upsilon\langle \cS \rangle \models \alpha). $
We can rewrite $\widetilde{A}^{(q)}_{\Upsilon}$ as $\widetilde{A}^{(q)}_{\Upsilon} =   \frac{1}{q} \sum_{i=1}^{q} \frac{1}{\lfloor n/k \rfloor} \sum_{\cS \in \mathbf{Y}_{i}^{(q)}} \mathds{1}(\Upsilon\langle \cS \rangle \models \alpha)$.

%From Lemma \ref{lemma:hoeffding1} we know that we can bound the probability that each of the $q$ terms $\frac{1}{\lfloor n/k \rfloor} \sum_{\cS \in \mathbf{Y}_{i}^{(q)}} \mathds{1}(\Upsilon\langle \cS \rangle \models \alpha)$ deviates from ${A}_\aleph$ by more than $\varepsilon$ by $2 \exp\left(-2 \left\lfloor \frac{n}{k} \right\rfloor \varepsilon^2 \right)$.
Then we can use the following trick (Hoeffding \cite{hoeffding1963probability}, Section 5) based on application of Jensen's inequality and Markov's inequality: If $T = a_1 \cdot T_1 + a_2 \cdot T_2 + \dots + a_q \cdot T_n$, where $a_i \geq 0$ and $\sum_{i=1}^q a_i = 1$, then, for any $h > 0$,
$P[T \geq \varepsilon] \leq \sum_{i=1}^n a_i \cdot \expect{\exp{\left( h (T_i - \varepsilon) \right)}}$.
Note that the $T_i$'s do not have to be independent.
Next, using Hoeffding's lemma (Lemma 1 in \cite{hoeffding1963probability}), if $a_i = 1/q$ and each of the terms $T_i$ is a sum of independent random zero-mean variables $X_{j}^{(i)}$ such that $P[a \leq X_{j}^{(i)} \leq b]=1$ and $b-a \leq 1$, then we get:
\begin{multline*}
P[T \geq \varepsilon] \le \sum_{i=1}^q \frac{1}{q} \cdot \expect{\exp{\left( h (T_i - \varepsilon) \right)}} \\
 \le e^{-h\varepsilon} \exp\left(\frac{m \cdot h^2}{8} \right) =  \exp\left(-h\varepsilon + \frac{m \cdot h^2  }{8} \right)
\end{multline*}
where $m$ denotes the number of summands of $T_i$ (which, in our case, is the same for all $T_i$'s). Note that this function achieves its minimum at $h = \frac{4\varepsilon}{m}.$
We set $T_i :=  \sum_{\cS \in \mathbf{Y}_{i}^{(q)}} \left( \mathds{1}(\Upsilon\langle \cS \rangle \models \alpha) - A_\aleph \right)$ (note that $\expect{T_i} = 0$ and $m = \lfloor n/k \rfloor$).
Thus, we get $P[\left\lfloor \frac{n}{k} \right\rfloor \cdot ( \widetilde{A}^{(q)}_{\Upsilon} - A_\aleph ) \geq \varepsilon] \leq \exp\left(-2  \varepsilon^2 / \left\lfloor \frac{n}{k} \right\rfloor \right),$ and finally
$$P[\widetilde{A}^{(q)}_{\Upsilon} - A_\aleph \geq \varepsilon] \leq \exp\left(-2 \left\lfloor \frac{n}{k} \right\rfloor \varepsilon^2 \right),$$
symmetrically also $P\left[A_\aleph - \widetilde{A}^{(q)}_{\Upsilon} \geq \varepsilon \right] \leq \exp\left(-2 \left\lfloor \frac{n}{k} \right\rfloor \varepsilon^2 \right)$,
and, using union bound, we get
$$P[|\widetilde{A}^{(q)}_{\Upsilon} - A_\aleph| \geq \varepsilon] \leq 2 \exp\left(-2 \left\lfloor \frac{n}{k} \right\rfloor \varepsilon^2 \right).$$
\noindent It follows from the strong law of large numbers (which holds for any $\Upsilon$) that $P[\lim_{q \rightarrow \infty}\widetilde{A}^{(q)}_{\Upsilon} = \widehat{A}_\Upsilon] = 1$. Since $q$ was arbitrary, the statement of the proposition follows.
\end{proof}

\noindent As the next theorem shows, the above result can be generalized to the case where we need to bound the difference between the estimations obtained from two samples.

\begin{theorem}\label{prop:two-sample}
Let $\aleph = (\cA_\aleph, \cC_\aleph)$ be an example and let $0 \leq n,u \leq |\cC_\aleph|$ and $0 \leq k \leq n$ be integers. Let $\mathcal{C}_\Upsilon$ and $\mathcal{C}_\Gamma$ be sampled uniformly from all size-$n$ and size-$u$ subsets of $\mathcal{C}_\aleph$ and let $\Upsilon = \aleph\langle \cC_\Upsilon \rangle, \Gamma = \aleph\langle \cC_\Gamma \rangle$. Let $\alpha$ be a closed and constant-free formula.
Let $\widehat{A}_\Upsilon = Q_{\Upsilon,k}(\alpha)$, $\widehat{A}_\Gamma = Q_{\Gamma,k}(\alpha)$,
and let $A_{\aleph} = Q_{\aleph,k}(\alpha)$. Then we have
$P[ \widehat{A}_\Upsilon - \widehat{A}_\Gamma \geq \varepsilon ] \leq \exp\left(\frac{-2\varepsilon^2}{1/\lfloor n/k\rfloor + 1/\lfloor u/k\rfloor}  \right)$,
 %\right)$\nb{This follows by symmetry from the first inequality, so I wouldn't mention this second inequality.},
and
$P\left[ \left| \widehat{A}_\Upsilon - \widehat{A}_\Gamma \right| \geq \varepsilon \right] \leq 2  \exp\left(\frac{-2\varepsilon^2}{1/\lfloor n/k\rfloor + 1/\lfloor u/k\rfloor}  \right).$
\end{theorem}
\begin{proof}
See the appendix.
\end{proof}

\noindent We note that the concentration inequality derived in Theorem \ref{prop:u-concentration} improves upon a concentration inequality derived in \cite{lovasz2012large} (Chapter 10) that contains $n/k^2$ (in our notation) instead of $\lfloor n/k \rfloor$ in the exponential.\footnote{This is essentially due to the fact that we use Hoeffding's decomposition whereas Lovasz relies on Azuma's inequality, leading to a looser bound compared to our bound.} %In addition, we could not use Lovasz' approach to obtain any counterpart of Theorem \ref{thm:realizable} \nb{Something is wrong here, as this theorem is only introduced below. Also, is there anything more you can say? Is this surprising (if so, maybe start the sentence with ``Surprisingly'')}.

%The second inequality that we prove in turn only involves
Next we prove an inequality for the special case where the probability of a formula $\alpha$ on $\Upsilon$ is $0$. Since we can also take negations of formulas, this theorem will be useful
%when we have
to prove bounds for formulas that are perfectly accurate on training data. As the following theorem shows, in this case we obtain stronger guarantees, where we have $\varepsilon$ instead of $\varepsilon^2$ in the exponential.

\begin{theorem}\label{thm:realizable}
Let $\aleph = (\cA_\aleph, \cC_\aleph)$ be an example and let $0 \leq n \leq |\cC_\aleph|$ and $0 \leq k \leq n$ be integers. Let $\mathcal{C}_\Upsilon$ be sampled uniformly from all size-$n$ subsets of $\mathcal{C}_\aleph$ and let $\Upsilon = \aleph\langle \cC_\Upsilon \rangle$. Let $\alpha$ be a closed and constant-free formula and let $\cC_k$ denote all size-$k$ subsets of $\cC_\Upsilon$. Let
$\widehat{A}_\Upsilon = Q_{\Upsilon,k}(\alpha)$
% $$\widehat{A}_\Upsilon = \left( \begin{array}{c} n \\ k \end{array} \right)^{-1} \sum_{\cS \in \cC_k} \mathds{1}(\Upsilon\langle \cS \rangle \models \alpha)$$
and let
$A_{\aleph} = Q_{\aleph,k}(\alpha) \geq \varepsilon.$
Then we have
$$P\left[\widehat{A}_\Upsilon = 0 \right] \leq\exp\left(- \left\lfloor n/k \right\rfloor \varepsilon \right).$$
\end{theorem}
\begin{proof}
Let $\mathbf{Y}$ be sampled as in Lemma \ref{lemma:crazy1} (i.e.\ $\mathbf{Y}$ is sampled only using $\Upsilon$ and not directly $\aleph$). Then using Lemma \ref{lemma:crazy1} we know that the elements of $\mathbf{Y}$ are distributed like $\lfloor n/k \rfloor$ independent samples (size-$k$ subsets) from $C_\aleph$. Hence we can bound the probability $P[A_\Upsilon = 0 ] \leq (1-\varepsilon)^{\lfloor n/k \rfloor} \leq \exp{\left( - \lfloor n/k \rfloor \varepsilon \right)}$. Obviously, adding the rest of the information from size-$k$ subsets of $C_\Upsilon$ that are not contained in $\mathbf{Y}$ cannot increase the bound.
\end{proof}

\subsection{ZERO TRAINING ERROR CASE}\label{sec:realizable}

We start by proving a bound for the realizable (i.e.\ zero training error) case.

\begin{theorem}\label{thm:realizable-expected}
Let $\aleph$, $\Upsilon$, $\Gamma$, $n$, $u$ and $\kappa$ be as in Definition \ref{def:inductive_setting} (i.e. $\Upsilon$ and $\Gamma$ are sampled from $\aleph$ and $n, u$ are sizes of $\Upsilon$'s and $\Gamma$'s domains).
Let $\mathcal{H}$ be a finite hypothesis class of constant-free formulas.
Let $\mathcal{F}(\Gamma,\Phi)$ denote the set of all ground literals of a predicate $p/a$ that are $k$-entailed by $\{\kappa(\Gamma)\} \cup \Phi$ but are false in $\Gamma$.\footnote{Note that here, as well as in the rest of the theorems in the paper, $\mathcal{F}(\Gamma,\Phi)$ is a set-valued random variable.}
With probability at least $1-\delta$, the following holds for all $\Phi \in \mathcal{H}$ that satisfy $Q_{\Upsilon,k}(\Phi) = 1$:
$$\expect{|\mathcal{F}(\Gamma,\Phi)|} \leq \frac{\ln{|\mathcal{H}|}+\ln{1/\delta}}{\lfloor n/k \rfloor} u^k k^a.$$
\end{theorem}
\begin{proof}
It follows from the linearity of expectation and from Proposition \ref{prop:simplebound} that, for any $\Phi$,
$\expect{|\mathcal{F}(\Gamma,\Phi)|} \leq (1-Q_{\aleph,k}(\Phi)) u^k k^a.$
Next, it follows from Theorem \ref{thm:realizable} and from the union bound taken over all $\Phi \in \mathcal{H}$ that the probability that there exists $\Phi \in \mathcal{H}$ such that $Q_{\Upsilon,k}(\Phi) = 1$ and $\varepsilon \leq 1-Q_{\aleph,k}(\Phi)$ is at most $|\mathcal{H}| \cdot \exp{\left( - \lfloor n/k \rfloor \varepsilon \right)}$.
If $\varepsilon \geq \frac{\ln{|\mathcal{H}|}+\ln{1/\delta}}{\lfloor n/k \rfloor}$ then $|\mathcal{H}| \cdot \exp{\left( - \lfloor n/k \rfloor \varepsilon \right)} \leq \delta$.
Hence, with probability at least $1-\delta$, the following holds for all $\Phi \in \mathcal{H}$ such that $Q_{\Upsilon,k}(\Phi) = 1$:
$\expect{|\mathcal{F}(\Gamma,\Phi)|} \leq \frac{\ln{|\mathcal{H}|}+\ln{1/\delta}}{\lfloor n/k \rfloor} u^k k^a.$
\end{proof}

\subsection{GENERAL CASE}\label{sec:non-zero}

Next we prove a bound for the general case when the training error is non-zero.

\begin{theorem}\label{thm:expected} Let $\aleph$, $\Upsilon$, $\Gamma$, $n$, $u$ and $\kappa$ be as in Definition \ref{def:inductive_setting}  (i.e. $\Upsilon$ and $\Gamma$ are sampled from $\aleph$ and $n, u$ are sizes of $\Upsilon$'s and $\Gamma$'s domains).
Let $\mathcal{H}$ be a finite hypothesis class of constant-free formulas.
Let $\mathcal{F}(\Gamma,\Phi)$ denote the set of all ground literals of a predicate $p/a$ that are $k$-entailed by $\{\kappa(\Gamma)\} \cup \Phi$ but are false in $\Gamma$.
With probability at least $1-\delta$, for all $\Phi \in \mathcal{H}$:
\begin{align*}
    \expect{|\mathcal{F}(\Gamma,\Phi)|} \leq \left(1-Q_{\Upsilon,k}(\Phi)+\sqrt{\frac{\ln{\left(\frac{|\mathcal{H}|}{\delta}\right)}}{2 \lfloor n/k \rfloor}} \right) u^k k^a.
\end{align*}
\end{theorem}
\begin{proof}
First, as in the proof of Theorem \ref{thm:realizable-expected}, we find that, for any $\Phi \in \mathcal{H}$,
$\expect{|\mathcal{F}(\Gamma)|} \leq (1-Q_{\aleph,k}(\Phi)) u^k k^a.$
Next, it follows from Theorem \ref{prop:u-concentration} and from union bound that
% \begin{multline*}
%     P\left[ \exists \Phi \in \mathcal{H} : Q_{\Upsilon,k}(\Phi)- Q_{\aleph,k}(\Phi) \geq \varepsilon \right] \\
%     \leq |\mathcal{H}| \exp{\left( - 2 \lfloor n/k \rfloor \varepsilon^2 \right)}.
% \end{multline*}
$
    P\left[ \exists \Phi \in \mathcal{H} : Q_{\Upsilon,k}(\Phi)- Q_{\aleph,k}(\Phi) \geq \varepsilon \right]
    \leq |\mathcal{H}| \exp{\left( - 2 \lfloor n/k \rfloor \varepsilon^2 \right)}.
$
It follows that
{\small
$$P\left[\exists \Phi \in \mathcal{H} : Q_{\Upsilon,k}(\Phi) \geq Q_{\aleph,k}(\alpha)+\sqrt{\frac{\ln{\left(|\mathcal{H}|/\delta\right)}}{2\lfloor n/k \rfloor}} \right] \leq \delta.$$}
The theorem then follows straightforwardly from the above and from Proposition \ref{prop:simplebound}.
\end{proof}

The previous two theorems provided bounds on the expected number of errors on the sampled test examples. The next theorem is different in that it provides a bound on the actual number of errors.

\begin{theorem}\label{thm:epsilon-delta}
Let $\aleph$, $\Upsilon$, $\Gamma$, and $\kappa$ be as in Definition \ref{def:inductive_setting}  (i.e. $\Upsilon$ and $\Gamma$ are sampled from $\aleph$ and $n, u$ are sizes of $\Upsilon$'s and $\Gamma$'s domains).
Let $\mathcal{H}$ be a finite hypothesis class of constant-free formulas.
Let $\mathcal{F}(\Gamma,\Phi)$ denote the set of all ground literals of a predicate $p/a$ that are $k$-entailed by $\{\kappa(\Gamma)\} \cup \Phi$ but are false in $\Gamma$.
With probability at least $1-\delta$, for all $\Phi \in \mathcal{H}:$
\begin{multline*}
    |\mathcal{F}(\Gamma,\Phi)|
    \leq \Bigg( 1-Q_{\Upsilon,k}(\Phi) + \\
    \sqrt{\frac{(\lfloor n/k \rfloor + \lfloor u/k \rfloor) \ln{\left(2|\mathcal{H}|/\delta\right)}}{2\lfloor n/k \rfloor  \lfloor u/k \rfloor}} \Bigg) u^k k^a \\
    \leq \Bigg( 1-Q_{\Upsilon,k}(\Phi) +
    \sqrt{\frac{ \ln{\left(2|\mathcal{H}|/\delta\right)}}{\min(\lfloor n/k \rfloor,  \lfloor u/k \rfloor)}} \Bigg) u^k k^a.
\end{multline*}
\end{theorem}
\begin{proof}
%(i)
Let us denote $\widehat{A} = Q_{\Upsilon,k}(\Phi)$, $\widehat{B} = Q_{\Gamma,k}(\Phi)$.
%, and $\widehat{C} = Q_{\aleph,k}(\Phi)$.
%First, we need to bound the probability $P[\widehat{A}-\widehat{B} \geq \varepsilon]$. It is easy to see that
%{\small $$P[\widehat{A}-\widehat{B} \geq \varepsilon] \leq 2 \cdot \max{\{ P[\widehat{A}-\widehat{C} \geq \varepsilon/2], P[\widehat{C}-\widehat{B} \geq \varepsilon/2]\}}.$$}
%Then,
Using Theorem \ref{prop:two-sample} and the union bound over $\Phi \in \mathcal{H}$, we get
%$P[\exists \Phi \in \mathcal{H}:|\widehat{A}-\widehat{B}| \geq \varepsilon] \leq 2 |\mathcal{H}| \exp{ \left(  \frac{- 2 \varepsilon^2 \lfloor n/k\rfloor \lfloor u/k\rfloor}{\lfloor n/k\rfloor + \lfloor u/k\rfloor} \right)}.$
{\small $$P[\exists \Phi \in \mathcal{H}:|\widehat{A}-\widehat{B}| \geq \varepsilon] \leq 2 |\mathcal{H}| \exp{ \left(  \frac{- 2 \varepsilon^2 \lfloor n/k\rfloor \lfloor u/k\rfloor}{\lfloor n/k\rfloor + \lfloor u/k\rfloor} \right)}.$$}
Solving the above for $\varepsilon$ that achieves the $1-\delta$ bound, we obtain that, with probability at least $1-\delta$, we have for all $\Phi \in \mathcal{H}$:
$|\widehat{A}-\widehat{B}| \leq\sqrt{\frac{(\lfloor n/k \rfloor + \lfloor u/k \rfloor) \ln{\left(2|\mathcal{H}|/\delta\right)}}{2\lfloor n/k \rfloor  \lfloor u/k \rfloor}}.$
Hence, with probability at least $1-\delta$, for all $\Phi \in \mathcal{H}$
it holds $1-Q_{\Gamma,k}(\Phi) \leq 1-Q_{\Upsilon,k}(\Phi) + \sqrt{\frac{(\lfloor n/k \rfloor + \lfloor u/k \rfloor) \ln{\left(2|\mathcal{H}|/\delta\right)}}{2\lfloor n/k \rfloor  \lfloor u/k \rfloor}}$.
 The validity of the theorem then follows from the above and from Proposition \ref{prop:simplebound} and the fact that $\frac{ab}{a+b} \geq \frac{\min(a,b)}{2}$ for any nonnegative $a$ and~$b$.
\end{proof}

\subsection{BOUNDS FOR VOTING ENTAILMENT}\label{sec:voting-bounds}

Next we prove a bound for voting entailment, which, unsurprisingly, is tighter than the respective bound for $k$-entailment.

\begin{theorem}\label{thm:voting}
Let $k$ be an integer and $\gamma \in [0;1]$. Let further $\aleph$, $\Upsilon$, $\Gamma$ and $\kappa$ be as in Definition \ref{def:inductive_setting}  (i.e. $\Upsilon$ and $\Gamma$ are sampled from $\aleph$ and $n, u$ are sizes of $\Upsilon$'s and $\Gamma$'s domains).
Let $\mathcal{H}$ be a finite hypothesis class of constant-free formulas.
Let $\mathcal{F}(\Gamma,\Phi)$ denote the set of all ground literals of a predicate $p/a$ that are entailed by voting from $\{\kappa(\Gamma)\} \cup \Phi$ with parameters $k$ and $\gamma$ but are false in $\Gamma$.  Then, with probability at least $1-\delta$, for all $\Phi \in \mathcal{H}$:
\begin{multline*}
    |\mathcal{F}(\Gamma)| \leq \\ \left( 1-Q_{\Upsilon,k}(\Phi) + \sqrt{\frac{ \ln{\left(2|\mathcal{H}|/\delta\right)}}{\min{\{\lfloor u/k \rfloor, \lfloor n/k \rfloor\}}}} \right) \frac{u^{a} k^a}{\gamma}.
\end{multline*}
\end{theorem}
\begin{proof}
This follows from the same reasoning as in the proof of Theorem \ref{thm:epsilon-delta}, which gives us the bound on the difference of $Q_{\Upsilon,k}(\Phi)$ and $Q_{\Gamma,k}(\Phi)$, combined with Theorem \ref{prop:voting}.
\end{proof}

\begin{observation}
The fraction of ``wrong'' ground $p/a$ literals does not grow with increasing test-set size ($u$), since, by rewriting the bound from Theorem \ref{thm:voting}, we get, with probability at least $1-\delta$, for all $\Phi \in \mathcal{H}$:

{\small $$\frac{|\mathcal{F}(\Gamma)|}{u^a} \leq \left( 1-Q_{\Upsilon,k}(\Phi) + \sqrt{\frac{\ln{\left(2|\mathcal{H}|/\delta\right)}}{\min{\{\lfloor u/k \rfloor, \lfloor n/k \rfloor\}}}} \right) \frac{k^a}{\gamma}.$$}
\end{observation}

We note here that one can also easily obtain counterparts of Theorems \ref{thm:realizable-expected} and \ref{thm:expected} for voting entailment.

%\section{SUMMARY OF POSITIVE AND NEGATIVE RESULTS}
\section{SUMMARY OF RESULTS}\label{sec:results}

In this section we discuss positive and negative results that follow from the theorems presented in the preceding sections. Here, bounds are considered vacuous if they are not lower than the total number of ground literals. We first focus on $k$-entailment in Sections \ref{sec61}--\ref{sec63}, and then discuss the results for voting entailment in Section \ref{sec64}. Finally, we also make a connection to MAP-entailment in Section~\ref{secMAP}.
%they are not predict tighter than the number of all ground

\subsection{SMALL TEST EXAMPLES}\label{sec61}

One case where we have non-vacuous bounds for the {\em expected} number of incorrectly predicted literals with $k$-entailment is when the domain of the test examples $\Gamma$ is small. Naturally a necessary condition is also that the given (or learned) theory $\Phi$ is sufficiently accurate. The only way to be confident that $\Phi$ is indeed sufficiently accurate, given that this accuracy needs to be estimated, is by estimating it on a sufficiently large training example. This is essentially what Theorems \ref{thm:realizable-expected} and \ref{thm:expected} imply.

Interestingly, this finding agrees with some experimental observations in the literature. For instance, it has been observed in \cite{kuvzelka2017induction} that classical reasoning in a relational setting close to ours worked well for small-size test-set evidence but was not competitive with other methods for larger evidence sizes. %In particular, this is the case for experiments reported there for prediction in a protein-protein interaction network.
%\nb{Removed the sentence about protein-protein interaction networks (seems too specific / not really relevant?).}
%\nb{Is this really warranted? We found that we performed better than MLNs for small evidence sizes, but this seems unrelated (e.g.\ being more about the fact that we can't distinguish between enough certainty levels to optimally take advantage of large evidence sets). Of course there is also the issue that with large evidence sets we are more likely to drown, which seems closer to the results from this paper, but still the link seems a bit far-fetched to me.}.
The analysis in the present paper thus sheds light on experimental observations like these.

Note that the bounds from Theorems \ref{thm:realizable-expected} and \ref{thm:expected} are for the {\em expected} value of the number of errors. Bounds on the actual number of errors are provided in Theorem \ref{thm:epsilon-delta}. In this case, to obtain non-vacuous bounds, we also need to require that the domain of the test example $\Gamma$ be sufficiently large. This is not unexpected, however, as it is a known property of statistical bounds for transductive settings (see e.g., \cite{tolstikhin2016minimax}) that the size of the test set affects confidence bounds, similarly to how the size of the $\Gamma$'s domain affects the bound in Theorem \ref{thm:epsilon-delta}.

\subsection{PREDICATES OF ARITY K}\label{sec62}

Another case where we have non-vacuous bounds for $k$-entailment is when the arity of the predicted literals is equal to the parameter $k$. In this case both the bounds for the expected error and for the actual error $|\mathcal{F}(\Gamma,\Phi)|$ are non-vacuous. This means that our results cover important special cases. One such special case is classical attribute-value learning when $k = 1$ and we represent attributes by unary predicates. Another case is link prediction when $k = 2$ and higher-arity versions thereof. In link prediction, we have rules such as, for instance,
$\forall X,Y : \textit{CoensFan}(X) \wedge \textit{CoensFilm}(Y) \Rightarrow \textit{likes}(X,Y).$
%This case is also relevant for predicting relations in knowledge graphs \steven{when} the predictions are based on properties of the vertices\nb{How so? Don't you have to mix unary and binary predicates in that case?} (entities), e.g.\ their vector-space embeddings, e.g. \cite{wang2014knowledge}. The positive results discussed here then also apply in these settings.

\subsection{REALIZABLE SETTING}\label{sec63}

We can get stronger guarantees when the given (or learned) theory $\Phi$ has zero training error. Keeping the fraction of the domain-sizes $|\cC_{\Gamma}|^{k-a}/|\cC_\Upsilon|$ small, Theorem \ref{thm:realizable-expected} implies non-vacuous bounds for predicates of arity $a$ for any size of the domain of $\Gamma$. Intuitively, this means that we can use theories that are completely accurate on training data for inference using $k$-entailment. However, the required size of the domain of the training example $\Upsilon$, to guarantee that we will not produce too many errors, grows exponentially with $k$ (for a fixed arity $a$) and polynomially with $|\cC_\Gamma|$.

\subsection{VOTING}\label{sec64}

When using voting entailment, we can always obtain non-trivial bounds by making $\gamma$ large; obviously this comes at the price of making the inferences more cautious. Voting entailment is a natural inference method in domains where one proof is not enough, i.e.\ where the support from several proofs is needed before we can be sufficiently confident in the conclusion; an example of such a domain is the well-known {\em smokers} domain, where knowing that one friend smokes does not provide enough evidence to conclude that somebody smokes; only if we have evidence of several smoker friends is the conclusion warranted that this person smokes.

\subsection{RELATIONSHIP TO MAP INFERENCE}\label{secMAP}

A popular approach to collective classification in relational domains is MAP-inference in Markov logic networks. Therefore a natural question is how this approach performs in our setting. Perhaps surprisingly, it might produce as many errors as classical logic reasoning in the examples from Section \ref{secClassicalErrors}, if the Markov logic network contains the same rules, all with positive weights, as we had in these examples. This is because MAP-inference will predict the same literals as classical logical inference when the rules from the Markov logic network are consistent with the given evidence.
%\nb{Don't we have to require that the weights are positive for this to hold?}. -> fixed
Thus, we can see that our guarantees for both $k$-entailment and voting entailment are better than guarantees one could get for MAP-inference. This is also in agreement with the well-known observations that,
%\nb{Do you have a reference? If not, I would rephrase (e.g. ``the well-known observation'').}
for instance, in the smokers domain, MAP inference often predicts everyone to be a smoker or everyone to be a non-smoker if there is only a small amount of evidence.% about the smoking behavior of the considered individuals.

\section{RELATED WORK}

%The main inspiration for the present work are
Our main inspiration comes from
the works on PAC-semantics by Valiant \cite{valiant_knowledge_infusion} and Juba \cite{juba}. Our work differs mainly in the fact that we have one large relational structure $\aleph$, and a training example $\Upsilon$ and a test example $\Gamma$, both sampled from $\aleph$, whereas it is assumed in these existing approaches
%the original work on PAC-semantics
that learning examples are sampled i.i.d. from some distribution. This has two important consequences. First, they could use statistical techniques developed for i.i.d. data whereas we had to first derive concentration inequalities for sampling without replacement in the relational setting. Second, since they only needed to bound the error on the independently sampled examples, they did not have to consider the number of incorrectly inferred facts. In contrast, in the relational setting that we considered here, the number of errors made on one relational example is the quantity that needs to be bounded. It follows that completely different techniques are needed in our case. Another difference is that, in their case, the training examples are also masked. In principle, we could modify our results to accommodate for masked examples by replacing ``accurate'' formulas by sufficiently-often ``witnessed'' formulas (see \cite{juba} for a definition).

Dhurandhar and Dobra \cite{dhurandhar2012distribution} derived Hoeffding-type inequalities for classifiers trained with relational data, but these inequalities, which are based on the restriction on the independent interactions of data points, cannot be applied to solve the problems considered in the present paper.
%In particular,
%Their setting is also more restricted than ours.
Certain other statistical properties of learning have also been studied for SRL models. For instance, Xiang and Neville \cite{xiang2011relational} studied consistency of estimation. However, guaranteeing convergence to the correct distribution does not mean that the model would not generate many errors when used, e.g., for MAP-inference. In \cite{xiang2011understanding}, they further studied errors in label propagation in collective classification. In their setting, however, the relational graph is fixed and one only predicts labels of vertices exploiting the relational structure for making the predictions.
Here we also note that it is not always possible or desirable in practice to sample sets of domain elements uniformly as we assumed to be the case in our analysis. Other sampling designs for relational data were studied, e.g. in \cite{ahmed2012network}. A study of PAC guarantees for such other sampling designs is left as a topic for future work.

There have also been works studying restricted forms of inference in a purely logical context, e.g.\ \cite{d2013semantics}. It is an interesting question for future work to find out which existing restricted inference systems would lead to non-vacuous error bounds in the relational setting.

%Finding out for which other sampling designs one could obtain PAC guarantees that would still be intuitive is an interesting topic, we leave it for future work.

\section{CONCLUSIONS}

We have studied the problem of predicting plausible missing facts in relational data, given a set of imperfect logical rules, in a PAC reasoning setting. As for the considered inference methods, one of our main objectives was for the inference methods to stay close to classical logic. The first inference method, $k$-entailment, is a restricted form of classical logic inference and hence satisfies this objective. The second inference method, voting entailment, is based on a form of voting that combines results from inferences made by $k$-entailment on subsets of the relational data. Importantly, the voting is not weighted which makes voting entailment easier to understand.
%\nb{We could say that this makes it easier to understand, but would weighting really make it further from classical logic? Maybe it is safer to remove this sentence?}.
We were able to obtain non-trivial bounds for the number of literals incorrectly predicted by a learned (or given) theory for both $k$-entailment and voting entailment. Probably the most useful results of our analysis lie in the identification of cases where the bounds for learning and reasoning in relational data are non-vacuous, which we discussed in detail in Section \ref{sec:results}.

There are many interesting directions in which one could extend the results presented in this paper. For instance, as practical means to improve the explainability of inferences made by voting entailment, we could first find representatives of isomorphism classes of ``proofs'' that are aggregated by voting entailment, and only show these to the user. Another direction is to extend the notion of implicit learning from~\cite{juba} into the relational setting. It would also be interesting to exploit explicit sparsity constraints and to study other sampling designs, although that might also turn out to be analytically less tractable than the setting considered in the present paper. Finally, although all bounds presented in this paper assume finite hypothesis classes, we note that it is also possible to extend our results to infinite hypothesis classes \cite{kuzelka2018ecmlpkdd}.

%\todo{mention that an interesting extension would be to consider explicit multiplicity constraints + implicit reasoning}

%\label{submission}

\noindent {\bf Acknowledgments}

OK's work was partially supported by the Research Foundation - Flanders (project G.0428.15). SS is supported by ERC Starting Grant 637277. JD is partially supported by the KU Leuven Research Fund  (C14/17/070,C22/15/015,C32/17/036), and FWO-Vlaanderen (SBO-150033). %YW is partially supported by Guangdong Shannon Intelligent Tech.\ co., Ltd.

\bibliography{reference}
\bibliographystyle{plain}

\ifappendix

\appendix

\section{OMITTED PROOFS}

\begin{proof}[Proof of Theorem \ref{prop:two-sample}]
First we define two auxiliary estimators $\widetilde{A}^{(q)}_{\Upsilon}$ and $\widetilde{A}^{(q)}_{\Gamma}$.
Let $\mathbf{Y}^{(q)}$ be a vector of $\lfloor n/k \rfloor \cdot q$ size-$k$ subsets of $\cC_\Upsilon$ where the subsets of $\cC_{\Upsilon}$ in each of the $q$ non-overlapping size-$\lfloor n/k \rfloor$ segments $\mathbf{Y}_1^{(q)}, \mathbf{Y}_2^{(q)},\dots,\mathbf{Y}_q^{(q)}$ of $\mathbf{Y}^{(q)}$ are sampled in the same way as the elements of the vector $\mathbf{Y}$ in Lemma \ref{lemma:crazy1}, all with the same $\cC_\Upsilon$ (i.e. $\mathbf{Y}^{(q)}$ is the concatenation of the vectors $\mathbf{Y}^{(q)}_1,\mathbf{Y}^{(q)}_2,\dots,\mathbf{Y}^{(q)}_q$).
Another vector $\mathbf{Z}^{(q)}$ which contains $\lfloor u/k\rfloor\cdot q$ size-$k$ subsets of $\cC_{\Gamma}$ is sampled in the same way. Note that $\mathbf{Z}^{(q)}$ is independent of $\mathbf{Y}^{(q)}$.
Let us define
\begin{align*}
\widetilde{A}^{(q)}_{\Upsilon} &= \frac{1}{q \cdot \lfloor n/k \rfloor} \sum_{\cS \in \mathbf{Y}^{(q)}} \mathds{1}(\Upsilon\langle \cS \rangle \models \alpha) \text{ and,}\\ \widetilde{A}^{(q)}_{\Gamma} &=   \frac{1}{q \cdot \lfloor u/k \rfloor} \sum_{\cS \in \mathbf{Z}^{(q)}} \mathds{1}(\Gamma\langle \cS \rangle \models \alpha).
\end{align*}
% $$\widetilde{A}^{(q)}_{\Gamma} =   \frac{1}{q} \sum_{i=1}^{q} \frac{1}{\lfloor u/k \rfloor} \sum_{\cS \in \mathbf{Z}_{i}^{(q)}} \mathds{1}(\Gamma\langle \cS \rangle \models \alpha)$$
We can rewrite them as
\begin{align*}
    \widetilde{A}^{(q)}_{\Upsilon} =  \frac{1}{q} \sum_{i=1}^{q} \frac{1}{\lfloor n/k \rfloor} \sum_{\cS \in \mathbf{Y}_{i}^{(q)}} \mathds{1}(\Upsilon\langle \cS \rangle \models \alpha),\\
    \widetilde{A}^{(q)}_{\Gamma} =   \frac{1}{q} \sum_{i=1}^{q} \frac{1}{\lfloor u/k \rfloor} \sum_{\cS \in \mathbf{Z}_{i}^{(q)}} \mathds{1}(\Gamma\langle \cS \rangle \models \alpha).
\end{align*}
%From Lemma \ref{lemma:hoeffding1} we know that we can bound the probability that each of the $q$ terms $\frac{1}{\lfloor n/k \rfloor} \sum_{\cS \in \mathbf{Y}_{i}^{(q)}} \mathds{1}(\Upsilon\langle \cS \rangle \models \alpha)$ deviates from ${A}_\aleph$ by more than $\varepsilon$ by $2 \exp\left(-2 \left\lfloor \frac{n}{k} \right\rfloor \varepsilon^2 \right)$.
%Then we can use the following fact (Hoeffding \cite{hoeffding1963probability}, Section 5): If $T = a_1 \cdot T_1 + a_2 \cdot T_2 + \dots + a_q \cdot T_q$, $a_i \geq 0$, $\sum_{i=1}^q a_i = 1$ and each of the terms $T_i$ is a sum of independent random zero-mean variables $X_{j}^{(i)}$ such that $P[a \leq X_{j}^{(i)} \leq b]=1$ and $b-a \leq 1$, then, for any $h > 0$,
% Using the same arguments as in the proof of Theorem \ref{prop:u-concentration}, we have
% $$P[T \geq \varepsilon] \leq \sum_{i=1}^q a_i \cdot \expect{\exp{\left( h (T_i - \varepsilon) \right)}}$$
% where $a_i \geq 0$ and $\sum_{i=1}^1 a_i = 1$.
%Note that the $T_i$'s do not have to be independent and
% \begin{eqnarray*}
% \widetilde{A}^{(q)}_{\Upsilon} &-& \widetilde{A}^{(q)}_{\Gamma}\\
% &=& \frac{1}{q} \sum_{i=1}^{q} \frac{1}{\lfloor n/k \rfloor} \sum_{\cS \in \mathbf{Y}_{i}^{(q)}} [\mathds{1}(\Upsilon\langle \cS \rangle \models \alpha) - A_{\aleph}] \\ &-& \frac{1}{q} \sum_{i=1}^{q} \frac{1}{\lfloor u/k \rfloor} \sum_{\cS \in \mathbf{Z}_{i}^{(q)}} [\mathds{1}(\Gamma \langle \cS \rangle \models \alpha) - A_{\aleph}]
% \end{eqnarray*}
Let us denote $m_1 = \lfloor n/k \rfloor$, $m_2 = \lfloor u/k \rfloor$ and $T_i := \frac{1}{\lfloor n/k \rfloor} \sum_{\cS \in \mathbf{Y}_{i}^{(q)}} \left( \mathds{1}(\Upsilon\langle \cS \rangle \models \alpha) - A_\aleph \right) - \frac{1}{\lfloor u/k \rfloor} \sum_{\cS \in \mathbf{Z}_{i}^{(q)}} \left( \mathds{1}(\Gamma \langle \cS \rangle \models \alpha) - A_\aleph \right)$ (we note that $\expect{T_i} = 0$). Using the same arguments as in the proof of Theorem \ref{prop:u-concentration}, we obtain the following:
\begin{eqnarray*}
P[\widetilde{A}^{(q)}_{\Upsilon} &-& \widetilde{A}^{(q)}_{\Gamma} \geq \varepsilon] \\
% &=&P[\frac{1}{q} \sum_{i=1}^q [ \frac{1}{\lfloor n/k \rfloor} \sum_{\cS \in \mathbf{Y}_{i}^{(q)}} \mathds{1}(\Upsilon\langle \cS \rangle \models \alpha) ] \geq \varepsilon] \\
&\le& \sum_{i=1}^q \frac{1}{q} \cdot \expect{\exp{\left( h (T_i - \varepsilon) \right)}} \\
%  &=& \expect{\exp{\left( h T_i \right)}} \\
 &\le & e^{-h\varepsilon} \exp\left(\frac{  h^2}{8m_1} \right)\exp\left(\frac{  h^2}{8m_2} \right) \\
& = & \exp\left(-h\varepsilon + \frac{m_1+m_2   }{8 m_1 m_2}\cdot h^2 \right)
\end{eqnarray*}
The bound achieves its minimum at $h = \frac{4\varepsilon m_1 m_2}{m_1+m_2}.$

Thus, we get
% $P\left[\widetilde{A}^{(q)}_{\Upsilon} - \widetilde{A}^{(q)}_{\Gamma} \ge \varepsilon\right] \leq \exp\left(- \frac{2 \varepsilon^2}{\lfloor n/k \rfloor} \right).$
% It follows that
$$P[\widetilde{A}^{(q)}_{\Upsilon} - \widetilde{A}^{(q)}_{\Gamma} \geq \varepsilon] \leq \exp\left(\frac{-2\varepsilon^2}{1/\lfloor n/k\rfloor+1/\lfloor u/k\rfloor} \right),$$
symmetrically also $P[\widetilde{A}^{(q)}_{\Gamma} - \widetilde{A}^{(q)}_{\Upsilon} \geq \varepsilon] \leq \exp\left(\frac{-2\varepsilon^2}{1/\lfloor n/k\rfloor+1/\lfloor u/k\rfloor} \right)$,
and, using union bound, we get
$$P[|\widetilde{A}^{(q)}_{\Upsilon} - \widetilde{A}^{(q)}_{\Gamma} |\geq \varepsilon] \leq 2\exp\left(\frac{-2\varepsilon^2}{1/\lfloor n/k\rfloor+1/\lfloor u/k\rfloor} \right).$$
\noindent It follows from the strong law of large numbers (which holds for any $\Upsilon$ and $\Gamma$) that $P[\lim_{q \rightarrow \infty}\widetilde{A}^{(q)}_{\Upsilon} = \widehat{A}_\Upsilon \hbox{ and } \widetilde{A}^{(q)}_{\Gamma} = \widehat{A}_\Gamma] = 1$. Since $q$ was arbitrary, the statement of the proposition follows.
\end{proof}

\section{REPRESENTING CONSTANTS USING AUXILIARY PREDICATES}

In this paper we restricted ourselves to reasoning with theories that do not contain any constants. It is straightforward to extend our results to provide PAC-type bounds also for theories with constants by introducing auxiliary predicates. For instance, in the smokers domain, if we want to express that friends of Alice do not smoke, i.e. $\forall X : \textit{fr}(alice,X) \Rightarrow \neg \textit{sm}(X)$, then we may introduce an auxiliary predicate $\textit{friendOfAlice}/1$ and the rule becomes $\forall X : \textit{friendOfAlice}(X) \Rightarrow \neg \textit{sm}(X)$. We note here that it is not necessary to add auxiliary predicates explicitly in practice. We use auxiliary predicates just for theoretical purposes to explain how the results about PAC-reasoning derived in this paper can be applied when constants are allowed.

This also reveals interesting properties of the problem. For instance, in order to do non-trivial reasoning based on $k$-entailment with a theory consisting only of the rule $$\forall X,Y: \textit{sm}(X) \wedge \textit{fr}(X,Y) \Rightarrow \textit{sm}(Y)$$ we need $k \geq 2$. However, for the rule $$\forall X : \textit{friendOfAlice}(X) \Rightarrow \neg \textit{sm}(X)$$ we only need $k \geq 1$. Hence, for the derived PAC bounds, we can see that the expected number of errors made when using only the second rule grows as in the attribute-value case whereas the expected number of errors for the first rule may grow more quickly with the increasing size of the test examples (cf.\ Theorem \ref{thm:expected}).

\fi
\end{document}